\documentclass[pdflatex,sn-apa]{sn-jnl}

\usepackage{bookmark}  
\usepackage{graphicx}%
\usepackage{multirow}%
\usepackage{amsmath,amssymb,amsfonts}%
\usepackage{amsthm}%
\usepackage[title]{appendix}%
\usepackage{xcolor}%
\usepackage{textcomp}%
\usepackage{manyfoot}%
\usepackage{booktabs}%
\usepackage{algorithm}%
\usepackage{algorithmicx}%
\usepackage{algpseudocode}%
\usepackage{listings}%

\theoremstyle{thmstyleone}%
\newtheorem{theorem}{Theorem}
\theoremstyle{thmstyletwo}%

\theoremstyle{thmstylethree}%

\theoremstyle{thmstyleone}%
\newtheorem{lemma}{Lemma}%
\newtheorem{corollary}{Corollary}%

\begin{document}

\title[Optimal Spatio-Temporal Decoupling for BCP]{Optimal Spatio-Temporal Decoupling for Bayesian Conformal Prediction}


\author[1]{\fnm{Yu-Hsueh} \sur{Fang}}\email{d11725001@ntu.edu.tw}

\author*[1]{\fnm{Chia-Yen} \sur{Lee}}\email{chiayenlee@ntu.edu.tw}

\affil[1]{\orgdiv{Department of Information Management}, \orgname{National Taiwan University}, \orgaddress{\street{No. 1, Sec. 4, Roosevelt Rd., Da'an Dist.}, \city{Taipei}, \postcode{106216}, \country{Taiwan}}}


\begin{abstract}

Online conformal prediction must balance fast adaptation to distribution shift against stable coverage: feedback-driven methods react quickly but become volatile, while strongly discounted Bayesian methods lag and inflate intervals at tight coverage. We introduce \textbf{State-Adaptive Bayesian Conformal Prediction (SA-BCP)}, which forms the predictive quantile as a gated convex combination of long-term temporal inertia and local spatial evidence from a kernel density estimate, controlled by a single interpretable evidence threshold $K$. We establish three results: (i) asymptotic marginal validity of the resulting intervals up to a gate-controlled bias that vanishes as spatial evidence accumulates (exact under recurrent states); (ii) a closed-form expression for the MSE-optimal threshold, $K^*_{\mathrm{MSE}}=\alpha(1-\alpha)/M^{\mathcal{T}}$, trading the coverage-indicator (Bernoulli) variance against the temporal structural bias $M^{\mathcal{T}}$; and (iii) a rolling-origin procedure for selecting $K$ online---consistent under stationarity, with $O(\sqrt{T\log N})$ regret against the best fixed $K$ and, for a segmented variant, a sublinear dynamic-regret bound under sublinearly many ($B_T=o(T)$) threshold shifts. Across four financial-volatility and weather datasets, three target coverage levels, and eight baselines, SA-BCP attains at-or-above-nominal coverage in most settings while producing substantially sharper intervals---up to roughly $3\times$ lower Winkler score than discounted Bayesian CP at the tightest coverage---and a coverage-matched audit confirms these efficiency gains are not an artifact of under-coverage. We disclose our principal limitation: a volatility-specialized CF-GARCH competitor remains more efficient on its home volatility-base series, though it does not transfer across domains.
\end{abstract}


\maketitle
\section{Introduction}
\label{sec:introduction}
Uncertainty quantification in non-stationary time series remains a fundamental challenge \citep{Hullermeier21}. Conformal Prediction (CP) \citep{Vovk05} gives finite-sample marginal coverage under exchangeability, but real data streams violate this through gradual drift and abrupt regime shifts, and adapting CP to them faces a persistent \textbf{adaptation--stability tradeoff}. Feedback-driven methods---ACI \citep{Gibbs21}, AgACI \citep{Zaffran22}, DtACI \citep{Gibbs24}---over-optimize during stable regimes and under-cover during abrupt volatility clustering; temporally discounted Bayesian CP \citep{Zhang2024} avoids collapse but incurs structural lag, producing bloated intervals after a shock subsides, especially at the tight 95\% level relevant to operational risk thresholds in finance and weather forecasting.

State-conditional calibration \citep{Chen24, Jiang24} exploits the fact that seemingly unprecedented shifts often recur as variants of historical states, but purely spatial kernel weighting collapses the effective sample size and under-covers when matches are sparse. Neither pure temporal nor pure spatial weighting suffices: the two should be \emph{decoupled and combined adaptively}. We propose \textbf{State-Adaptive Bayesian Conformal Prediction (SA-BCP)}, which gates long-term temporal inertia with spatial kernel-density evidence through a single interpretable threshold $K$, isolating temporal drift from spatial recurrence and treating each with structurally appropriate evidence aggregation.

\textbf{Contributions.} Our work makes the following contributions:

\begin{enumerate}
    \item \textbf{A spatio-temporal decoupled CP framework} with an interpretable evidence threshold $K$, base-model agnostic across GARCH, EWMA, AR(1), and Ridge predictors.
    \item \textbf{Rolling-origin $K$-selection with three guarantees:} consistency to the Winkler grid-optimum $\tilde{K}$, $O(\sqrt{T \log N})$ regret against the best fixed $K$, and segmented $O(\sqrt{T(B_T+1)})$ drift safety (Theorem~\ref{thm:rolling_consistency}).
    \item \textbf{Theoretical foundations:} approximate marginal validity under sublinear drift (Theorem~\ref{thm:marginal_validity}; exact under recurrent states, Corollary~\ref{cor:exact_recurrent}) and the closed-form optimum $K^*_{\mathrm{MSE}}=\alpha(1-\alpha)/M^{\mathcal{T}}$ (Theorem~\ref{thm:mse_optimal}).
    \item \textbf{Comprehensive evaluation} across eight baselines (including the strongest recent conditional-quantile methods, SPCI and KOWCPI), four datasets spanning financial volatility and weather, three coverage levels, and three assumption-violation stress scenarios, with block-bootstrap CIs (block 20, $B=1000$) and a coverage-matched audit separating raw efficiency from under-coverage artifacts.
\end{enumerate}

\textbf{Findings.} Across the 24 settings SA-BCP holds at-or-above-nominal coverage while remaining the sharpest \emph{calibrated} method in most cells; the temporal--spatial decoupling is necessary (it beats its spatial-only ablation in every setting) and the efficiency gains survive a coverage-matched audit, while CF-GARCH, our principal limitation, is sharper in some cells largely through under-coverage and does not transfer (full per-cell results in Section~\ref{sec:experiments}).

\textbf{Paper organization.} Section~\ref{sec:related_work} situates SA-BCP within the CP literature. Section~\ref{sec:methodology} presents the algorithm including rolling-origin $K$-selection. Section~\ref{sec:theory} establishes the theoretical guarantees. Section~\ref{sec:experiments} reports the empirical evaluation. Section~\ref{sec:discussion} discusses limitations and connections to concurrent work. All theorems are proved in the main text, where every headline result---the full nine-method benchmark, the coverage-matched audit, and the decoupling ablation---also appears.

\section{Related Work}
\label{sec:related_work}
\noindent\textbf{Non-exchangeable and online adaptive CP.} Classical CP \citep{Vovk05} gives finite-sample marginal coverage under exchangeability. To go beyond it, \citet{Tibshirani19} reweight by likelihood ratios for covariate shift and \citet{Barber23} establish coverage under bounded distribution shift via non-exchangeable CP (NExCP), using decay-weighted nonconformity quantiles; we use NExCP ($\rho=0.99$) as a primary distribution-free baseline. Recent work shows even unweighted split CP retains guarantees on some non-exchangeable series \citep{oliveira24, Barber26ALT}. A complementary thread adapts online by adjusting the miscoverage rate: ACI \citep{Gibbs21} updates $\alpha_t$ from the empirical coverage gap, with AgACI \citep{Zaffran22} aggregating over learning rates and DtACI \citep{Gibbs24} adding dynamic thresholds with regret bounds. These hold long-run marginal coverage but lag abrupt shifts---$\alpha_t$ adjusts slowly relative to volatility-clustering timescales---producing under-coverage during transitions, which Section~\ref{sec:experiments} quantifies against SA-BCP's proactive recognition.

\noindent\textbf{Bayesian, discounted, and GARCH-hybrid CP.} Bayesian CP (BCP) \citep{Zhang2024} exponentially discounts the empirical CDF, bounding dynamic regret against adversarial shifts; the discount $\beta$ trades memory against adaptation speed. BCP avoids ACI's coverage collapse but has two structural limits: \emph{structural lag} (extreme residuals keep inflating intervals after volatility normalizes) and uncalibrated tight-coverage intervals (over-covering 3--5pp at 95\% with $2$--$4\times$ widths). SA-BCP keeps BCP's collapse protection as a worst-case fallback while resolving the lag through spatial memory. For volatility specifically, conformalized GARCH (CF-GARCH) conformalizes studentized residuals $|Y_t-\hat{f}_t(X_t)|/\hat\sigma_t$ with a \emph{parametric GARCH plug-in} for the scale $\hat\sigma_t$ \citep{Bollerslev1986, Lei18}; tied to that parametric volatility model, it is a strong baseline on smooth financial series but does not transfer to non-volatility bases (AR, Ridge), whereas SA-BCP operates on absolute residuals with a base-model-agnostic mixture.

\noindent\textbf{State-conditional and conditional-quantile CP.} Since non-stationarity often manifests as recurring states, state-conditional methods calibrate locally: \citet{Chen24} use semantic-embedding states and \citet{Jiang24} localized quantile regression over nearby states, but pure spatial weighting collapses the effective sample size when matches are sparse and under-covers (quantified by our Localized-CP ablation in Section~\ref{sec:experiments}). A related line estimates a \emph{conditional quantile} of the score: SPCI \citep{xu23} via a quantile random forest on a residual window (extending EnbPI \citep{Xu21}) and KOWCPI \citep{lee25} via an optimally reweighted Nadaraya--Watson kernel. These are weighted empirical-quantile predictors like SA-BCP, and KOWCPI's optimal weights subsume the fixed-weight \citep{Barber23} and kernel-localized \citep{Jiang24} schemes of which our Localized ablation is a special case---so we do \emph{not} claim kernel weighting per se as novel. SA-BCP's contribution is instead structural: a gated convex mixture of a \emph{state-agnostic} discounted channel and a \emph{time-agnostic} spatial channel through one evidence threshold $K$ with closed-form optimum $K^*_{\mathrm{MSE}}=\alpha(1-\alpha)/M^{\mathcal{T}}$ (Theorem~\ref{thm:mse_optimal}) and rolling-origin guarantees (Theorem~\ref{thm:rolling_consistency}); the discounted fallback and the evidence gate have no counterpart in SPCI/KOWCPI, which target stationary, strongly-mixing series rather than validity under regime change.

\section{Methodology}
\label{sec:methodology}

\subsection{Problem Setup and Notation}
\label{sec:problem_setup}

Consider an online time-series forecasting setting where at each time step $t \in \{1, 2, \ldots, T\}$, we observe a feature vector $X_t \in \mathcal{X}$ and seek to predict a continuous target $Y_t \in \mathbb{R}$. We follow the standard inductive CP setup with a base point predictor $\hat{f}_t: \mathcal{X} \to \mathbb{R}$ that may itself be updated online (e.g., a Fast Online GARCH model, an EWMA scale estimator, or a Ridge regressor with online updates). The base predictor is treated as a black-box; SA-BCP is base-model agnostic.

Our goal is to construct a prediction interval $C_t(X_t) = [\hat{f}_t(X_t) - \hat{q}_t, \hat{f}_t(X_t) + \hat{q}_t]$, where the half-width $\hat{q}_t \geq 0$ is computed by SA-BCP, such that:
\begin{equation}
    \mathbb{P}(Y_t \in C_t(X_t)) \geq 1 - \alpha
    \label{eq:coverage_target}
\end{equation}
for a user-specified miscoverage rate $\alpha \in (0, 1)$. The corresponding non-conformity score is the model-agnostic absolute residual $E_t = |Y_t - \hat{f}_t(X_t)|$. We deliberately use the absolute residual rather than a volatility-studentized score (as in CF-GARCH): studentization sharpens intervals when the base predictor is itself a volatility model, but forfeits transfer to mean-model bases (AR, Ridge) and non-financial domains---a trade-off our experiments make explicit (Section~\ref{sec:experiments}). Throughout the paper, $\mathbb{I}(\cdot)$ denotes the indicator function. Let $\mathcal{F}_{t-1}$ be the $\sigma$-algebra generated by all observations up to time $t-1$. We adopt the convention that the spatial state $S_t$ and the spatial bandwidth $h$ are constructed using only $\mathcal{F}_{t-1}$-measurable quantities; observations $Y_t$ and base predictions $\hat{f}_t(X_t)$ enter the buffer only \emph{after} $Y_t$ is revealed at the end of step $t$. This protocol ensures that $\hat{q}_t$ is $\mathcal{F}_{t-1}$-measurable, a property essential for the martingale arguments in Section~\ref{sec:theory}.

\subsection{Spatio-Temporal Decoupling}
\label{sec:spatio_temporal_density}

The core mechanism of SA-BCP decouples two distinct sources of evidence about the conditional distribution of $E_t$: \emph{temporal recency} (recent observations are more relevant than old ones) and \emph{spatial similarity} (observations from similar past states are most relevant).

\noindent\textbf{Temporal density.} The temporal base density $D^{\mathcal{T}}_t$ implements discounted weighting inspired by Bayesian CP \citep{Zhang2024}:
\begin{equation}
    D^{\mathcal{T}}_t(i) = \beta^{t - 1 - i}, \quad \forall i < t
    \label{eq:temporal_density}
\end{equation}
where $\beta \in (0, 1)$ is the discount factor and $i$ indexes historical time. The total temporal weight is $D^{\mathcal{T}}_t = \sum_{i=0}^{t-1} \beta^{t-1-i} = (1 - \beta^t) / (1 - \beta)$, which converges to $1/(1-\beta)$ as $t \to \infty$. This component captures temporal inertia: recent residuals dominate the distribution estimate, but historical residuals retain non-zero weight, preventing the catastrophic forgetting that would occur under a sharp moving window.

\noindent\textbf{Spatial state and density.} We extract a local spatial state $S_t \in \mathbb{R}^d$ from a fixed-length window $w$ of past residuals and predictions. Specifically, $S_t = (E_{t-w}, \ldots, E_{t-1}, \hat{f}_{t-w}, \ldots, \hat{f}_{t-1})$, yielding $d = 2w$. We choose $w = 5$ for financial benchmarks and $w \in \{5, 7\}$ for weather benchmarks; the regime-recognition synthetic (Figure~\ref{fig:recognition}) uses $w = 1$ to isolate the recognition mechanism, while the theorem-validation synthetic uses $w = 5$.

We measure the similarity between the current state $S_t$ and each historical state $S_i$ using an anisotropic Gaussian kernel:
\begin{equation}
    D^{\mathcal{S}}_t(i) = \exp\left( -\frac{1}{2} \sum_{j=1}^d \left( \frac{S_{t,j} - S_{i,j}}{h_j} \right)^2 \right)
    \label{eq:spatial_density}
\end{equation}
where $S_{t,j}$ denotes the $j$-th component of $S_t$ and $h_j$ is the per-dimension bandwidth. The total spatial evidence is:
\begin{equation}
    D^{\mathcal{S}}_t = \sum_{i=0}^{t-1} D^{\mathcal{S}}_t(i)
    \label{eq:spatial_total}
\end{equation}

The per-dimension bandwidth $h_j$ is set online by a Scott-style rule and depends only on $\mathcal{F}_{t-1}$-measurable quantities; the exact rule and cold-start fallback are given in Section~\ref{sec:implementation}.

The crucial innovation of SA-BCP is how it adjudicates between the temporal base (recent inertia) and the spatial density (historical pattern memory). We introduce a single interpretable hyperparameter $K \geq 0$, the \emph{target matching number}, which acts as an evidence threshold. The spatial mixing weight $\pi^{\mathcal{S}}_t \in [0, 1]$ is defined as:
\begin{equation}
    \pi^{\mathcal{S}}_t = \frac{D^{\mathcal{S}}_t}{D^{\mathcal{S}}_t + K}
    \label{eq:mixture_gate}
\end{equation}

This formulation acts as an epistemic confidence gate with two regimes: in an \textbf{unprecedented state} ($D^{\mathcal{S}}_t \ll K$), $\pi^{\mathcal{S}}_t \to 0$ and the model conservatively defers to the temporal base, inheriting discounted-BCP's collapse protection \citep{Zhang2024} (Corollary~\ref{cor:collapse_protection}); in a \textbf{recognized regime} ($D^{\mathcal{S}}_t \gg K$), $\pi^{\mathcal{S}}_t \to 1$ and the model pivots to spatial memory, allowing proactive interval adjustment based on state recognition rather than reactive temporal feedback (Corollary~\ref{cor:exact_recurrent}).

The temporal mixing weight is $\pi^{\mathcal{T}}_t = 1 - \pi^{\mathcal{S}}_t$. The threshold $K$ trades spatial recognition against temporal stability; we derive the MSE-optimal $K^*_{\mathrm{MSE}}$ for this tradeoff in Theorem~\ref{thm:mse_optimal} and select $K$ from data in Section~\ref{sec:rolling_origin}. Figure~\ref{fig:recognition} illustrates the gate across a regime change.

\begin{figure}[ht]
    \centering
    \includegraphics[width=0.72\textwidth]{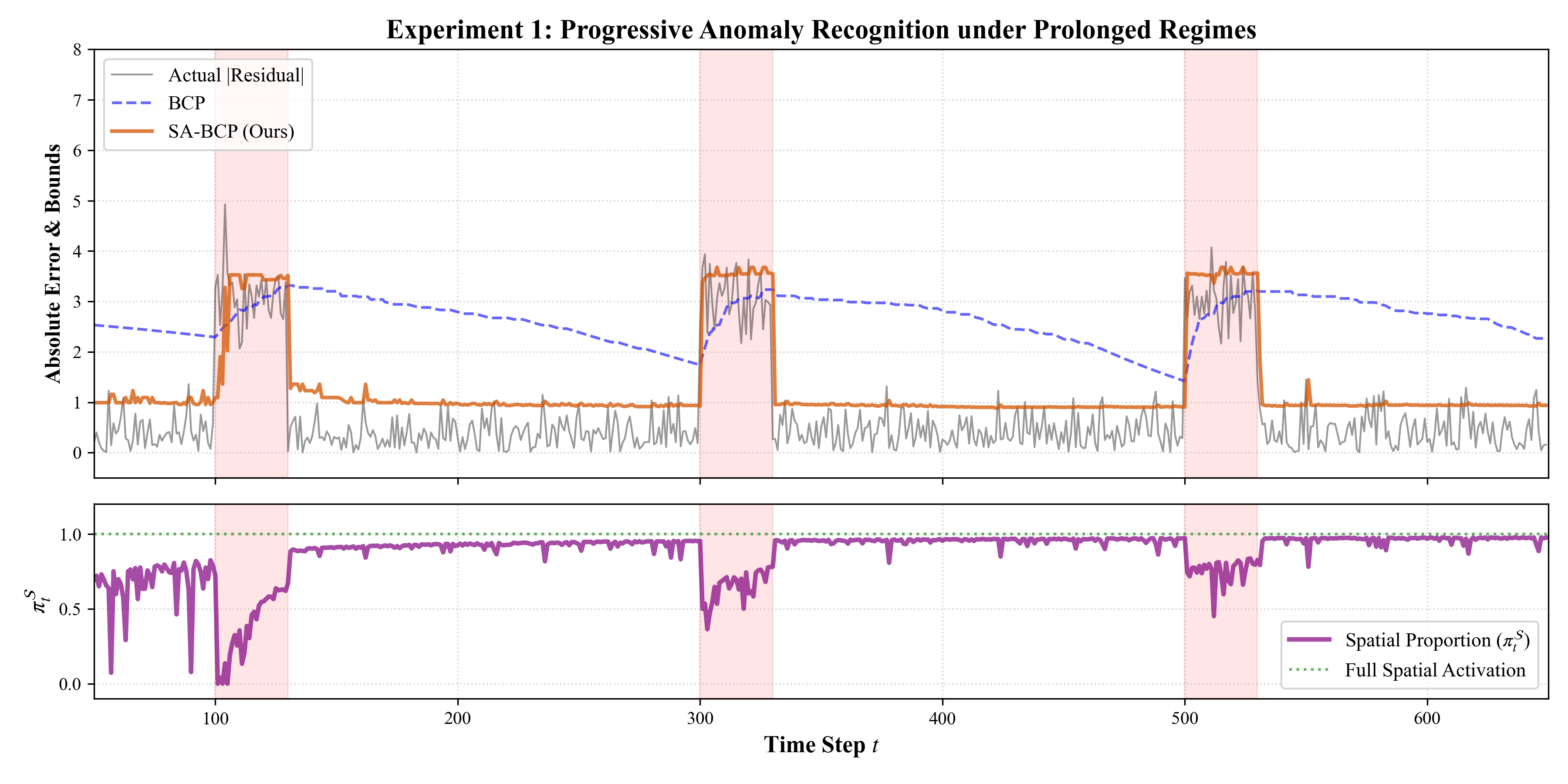}
    \caption{Progressive anomaly recognition (illustrative regime change). As a state recurs, the accumulated spatial density $D^{\mathcal{S}}_t$ raises the spatial weight $\pi^{\mathcal{S}}_t$, so SA-BCP pivots from the temporally discounted baseline to state-based recognition and adjusts the interval proactively rather than reacting only after coverage has been lost.}
    \label{fig:recognition}
\end{figure}

\subsection{State-Adaptive Quantile Construction}
\label{sec:quantile_construction}

Given the mixing weights $\pi^{\mathcal{S}}_t$ and $\pi^{\mathcal{T}}_t$, we construct the state-adaptive empirical CDF for the non-conformity scores by mixing two component CDFs:
\begin{align}
    \hat{F}^{\mathcal{S}}_t(r) &= \frac{\sum_{i=0}^{t-1} D^{\mathcal{S}}_t(i) \, \mathbb{I}(E_i \leq r)}{D^{\mathcal{S}}_t} \label{eq:spatial_cdf} \\
    \hat{F}^{\mathcal{T}}_t(r) &= \frac{\sum_{i=0}^{t-1} D^{\mathcal{T}}_t(i) \, \mathbb{I}(E_i \leq r)}{D^{\mathcal{T}}_t} \label{eq:temporal_cdf}
\end{align}

To handle the cold-start period when both spatial and temporal evidence are weak, we incorporate a uniform prior $\mathrm{Unif}(0, R)$, where $R$ is a fixed scale constant ($R = 15$ for financial data with returns scaled by 100; $R = 10$ for Jena temperature; $R = 3$ for synthetic data). The prior is weighted by $\lambda_t = 1/\sqrt{1+t}$, which decays to zero as $t$ grows so that the prior has no asymptotic influence. The final mixture CDF is:
\begin{equation}
    \hat{F}_t(r) = (1 - \lambda_t) \left[ \pi^{\mathcal{S}}_t \hat{F}^{\mathcal{S}}_t(r) + \pi^{\mathcal{T}}_t \hat{F}^{\mathcal{T}}_t(r) \right] + \lambda_t \cdot \min\!\left(\frac{r}{R}, 1\right)
    \label{eq:mixture_cdf}
\end{equation}

At each step $t$, we solve for the critical quantile $\hat{q}_t$ such that $\hat{F}_t(\hat{q}_t) = 1 - \alpha$, using Brent's root-finding algorithm \citep{brent2013} on the interval $[0, 1.5R]$. The cold-start uniform prior term $\lambda_t \cdot \min(r/R, 1)$ ensures strict monotonicity of $\hat{F}_t$ on $[0, R]$, guaranteeing a unique root. The prediction interval is then $C_t(X_t) = [\hat{f}_t(X_t) - \hat{q}_t, \hat{f}_t(X_t) + \hat{q}_t]$.

\subsection{Rolling-Origin \texorpdfstring{$K$}{K}-Selection}
\label{sec:rolling_origin}

While Theorem~\ref{thm:mse_optimal} characterizes the MSE-optimal $K^*_{\mathrm{MSE}}$ analytically as a ratio of unobservable quantities (the coverage-indicator variance and the temporal bias), practical deployment requires selecting $K$ from data alone. Selecting $K$ by sweeping over the full evaluation window and reporting the best Winkler score introduces in-sample selection bias that conflates method performance with hyperparameter tuning---a concern for online CP broadly, whose hyperparameters ($\gamma$ in ACI, $\beta$ in BCP, $K$ here) are often reported at hindsight-fixed values \citep{Zaffran22}. We address this by introducing a \emph{rolling-origin $K$-selection} protocol, which mirrors expanding-window cross-validation for time series and is, to our knowledge, the first systematic no-peeking selection rule for online adaptive CP.

\noindent\textbf{Protocol.}
We select $K$ by \emph{rolling-origin} validation: the stream is split into a warmup period and an evaluation horizon divided into consecutive blocks (calendar years in our experiments); for each block we pick $\hat{K} = \arg\min_{K \in \mathcal{K}_{\mathrm{grid}}} \mathrm{Winkler}(K; \mathcal{D})$ on all post-warmup data strictly preceding the block, then deploy it throughout that block (Algorithm~\ref{alg:rolling_origin}). The candidate grid is 12 logarithmically spaced values from $10^{-3}$ to $10^{3}$; the warmup size and block schedule are specified in Section~\ref{sec:exp_setup}.

\begin{algorithm}[ht]
\caption{Rolling-Origin $K$-Selection for SA-BCP}
\label{alg:rolling_origin}
\begin{algorithmic}[1]
\State \textbf{Input:} stream $\{(X_t, Y_t)\}_{t=1}^T$, warmup size $T_{\mathrm{warm}}$, evaluation blocks $\{B_1, \ldots, B_M\}$ partitioning $\{T_{\mathrm{warm}}{+}1, \ldots, T\}$, grid $\mathcal{K}_{\mathrm{grid}}$, miscoverage rate $\alpha$
\State \textbf{Output:} predictions $\{\hat{q}_t\}_{t=T_{\mathrm{warm}}+1}^T$, per-block selections $\{\hat{K}_m\}_{m=1}^M$
\State Run SA-BCP on warmup data to initialize buffers, bandwidths, and base model
\For{each evaluation block $B_m$, $m = 1, \ldots, M$}
    \State $\mathcal{D}_{<m} \leftarrow$ all post-warmup observations strictly preceding block $B_m$
    \For{each candidate $K \in \mathcal{K}_{\mathrm{grid}}$}
        \State Compute cumulative Winkler score $\widehat{W}(K; \mathcal{D}_{<m})$ using SA-BCP with this $K$
    \EndFor
    \State $\hat{K}_m \leftarrow \arg\min_{K \in \mathcal{K}_{\mathrm{grid}}} \widehat{W}(K; \mathcal{D}_{<m})$
    \State Deploy SA-BCP with $\hat{K}_m$ for all $t \in B_m$
\EndFor
\end{algorithmic}
\end{algorithm}

\noindent\textbf{Guarantees and efficiency.} Theorem~\ref{thm:rolling_consistency} gives three guarantees on this selector---consistency to the Winkler grid-optimum $\tilde{K}$, $O(\sqrt{T \log N})$ regret against the best fixed $K$, and segmented $O(\sqrt{T(B_T + 1)})$ drift safety (a worst-case bound, not a dominance guarantee). Empirically (Section~\ref{sec:experiments}), the no-peeking selection costs under $1\%$ of the full-sweep hindsight-oracle Winkler in every cell (maximum $0.71\%$).

\subsection{Implementation Details}
\label{sec:implementation}

SA-BCP is independent of the base predictor; only the residuals $E_t$ enter the algorithm (we use Fast Online GARCH(1,1), RiskMetrics EWMA, online AR(1)/RLS, and online Ridge across experiments; see Section~\ref{sec:experiments}). We fix $\beta = 0.99$ and window $w = 5$ for financial data ($w = 5$ for AR(1), $w = 7$ for Ridge on Jena); the candidate grid for $K$ is 12 logarithmically spaced values from $10^{-3}$ to $10^3$, and the uniform-prior cap $R$ is set per experiment (Section~\ref{sec:quantile_construction}). The anisotropic spatial bandwidth follows a Scott-style rule $h_j = \hat{\sigma}_j n^{-1/(d+4)}$ \citep{li2007}, with $\hat{\sigma}_j$ maintained by Welford's online algorithm and a cold-start default $h_j = 5.0$ for $n < 20$; all such quantities are $\mathcal{F}_{t-1}$-measurable. Each update is $O(n)$ in the accumulated memory $n$ (dominated by the kernel sum), sub-millisecond per step.

\section{Theoretical Guarantees}
\label{sec:theory}
Throughout, $\mathbb{I}(\cdot)$ is the indicator function, $\mathcal{F}_t$ the $\sigma$-algebra of observations through time $t$, and the measurability protocol of Section~\ref{sec:problem_setup} ($S_t$, $h_j$, $\hat{q}_t$ built only from $\mathcal{F}_{t-1}$-measurable quantities) is assumed.

\begin{theorem}[Approximate Asymptotic Marginal Validity]
\label{thm:marginal_validity}
Assume the residual process is weakly dependent ($\phi$-mixing with summable coefficients), the conditional distribution $P_t(E_t \mid \mathcal{F}_{t-1})$ has a bounded, Lipschitz-continuous density, and the total-variation drift is sublinear: $\sum_{t=1}^T d_{\mathrm{TV}}(P_t, P_{t+1}) = o(T)$. Let $\hat{q}_t$ be the SA-BCP quantile sequence solving $\hat{F}_t(\hat{q}_t) = 1 - \alpha$, where $\hat{F}_t$ is the mixture CDF defined in Equation~\ref{eq:mixture_cdf}. Under the measurability protocol of Section~\ref{sec:problem_setup}, for any $\alpha \in (0,1)$ and any finite $K > 0$:
\begin{equation}
    \limsup_{T \to \infty} \Big|\frac{1}{T} \sum_{t=1}^T \mathbb{I}(E_t \leq \hat{q}_t) - (1-\alpha)\Big| \;\le\; C_\beta \quad \text{a.s.},
\end{equation}
where the discounting-bias constant splits into a first-order aggregation term and a second-order curvature term,
\[
C_\beta \;\le\; \Big|\limsup_{T\to\infty}\tfrac1T\textstyle\sum_{t=1}^T \pi^{\mathcal{T}}_t\,\mathrm{bias}^{\mathcal{T}}_t\Big| \;+\; \frac{\sup_t|f'(r^*_t)|}{2\inf_t f(r^*_t)^2}\big(c\,(1-\beta)+o(1)\big),
\]
with $\mathrm{bias}^{\mathcal{T}}_t=\mathbb{E}[\hat F^{\mathcal{T}}_t(r^*_t)]-F_t(r^*_t)$ the signed temporal bias whose lag and state-aggregation components are described in Lemma~\ref{lemma:spatial_cdf}(ii) and $c$ an absolute constant. The first-order term is weighted by the gate $\pi^{\mathcal{T}}_t$ and vanishes in the spatial-sufficient regime $\pi^{\mathcal{T}}_t\to0$ (Corollary~\ref{cor:exact_recurrent}); the second-order term collects the discounting variance $\mathcal{O}(1-\beta)$ through the density curvature, written with $|f'(r^*_t)|$ since $f'<0$ at a right-tail quantile and with $\sup_t/\inf_t$ because $r^*_t$ varies in $t$. Annealing the discount, $\beta=\beta_T\uparrow1$ with $1-\beta_T=\omega(T^{-1})$ and $\sum_{t=1}^T d_{\mathrm{TV}}(P_t,P_{t+1})=o\big(T(1-\beta_T)\big)$, removes the second-order (discounting-variance) term and the temporal-lag part of $\mathrm{bias}^{\mathcal{T}}_t$, but not its state-aggregation part; full exactness---all of $C_\beta\to0$---is established separately in Corollary~\ref{cor:exact_recurrent} under spatial sufficiency.
\end{theorem}

\begin{proof}
Let $\epsilon_t=\mathbb{I}(E_t\le\hat q_t)-(1-\alpha)$. Since $\hat q_t$ is $\mathcal{F}_{t-1}$-measurable, $\mathbb{E}[\epsilon_t\mid\mathcal{F}_{t-1}]=F_t(\hat q_t)-(1-\alpha)$ with $F_t$ the true conditional CDF of $E_t$. Write $\epsilon_t=\xi_t+b_t$ with martingale difference $\xi_t:=\epsilon_t-\mathbb{E}[\epsilon_t\mid\mathcal{F}_{t-1}]$ ($|\xi_t|\le1$) and $\mathcal{F}_{t-1}$-measurable bias $b_t:=F_t(\hat q_t)-(1-\alpha)=(F_t-\hat F_t)(\hat q_t)$, using $\hat F_t(\hat q_t)=1-\alpha$. The partial sums $\sum_{t\le T}\xi_t$ form a bounded-increment martingale, so by Azuma--Hoeffding and Borel--Cantelli $\frac1T\sum_t\xi_t\to0$ a.s. For the bias, expand the plug-in quantile about the time-$t$ conditional quantile $r^*_t$ (where $F_t(r^*_t)=1-\alpha$): with $U_t:=\hat F_t(r^*_t)-(1-\alpha)$, $b_t=-U_t+\tfrac{f'(r^*_t)}{2f(r^*_t)^2}U_t^2+o(U_t^2)$. Decompose the estimation error into a mean-zero fluctuation and a bias, $U_t=\nu_t+\mathbb{E}[U_t]$ with $\nu_t:=U_t-\mathbb{E}[U_t]$; under $\phi$-mixing and the Lipschitz density the fluctuations have summable serial covariances, so $\frac1T\sum_t\nu_t\to0$ a.s.\ \citep{cesa2006prediction}. The mean splits over the channels, $\mathbb{E}[U_t]=\sum_{c}\pi^{c}_t\,\mathrm{bias}^{c}_t$: the spatial bias is $\mathcal{O}(\|h\|^2)\to0$ under the shrinking bandwidth, the prior weight $\lambda_t\to0$, and the temporal-lag part of $\mathrm{bias}^{\mathcal{T}}_t$ is $\mathcal{O}\big(\tfrac1T\sum_t d_{\mathrm{TV}}(P_t,P_{t+1})\big)\to0$ by sublinear drift. What remains is the temporal \emph{state-aggregation} bias, weighted by $\pi^{\mathcal{T}}_t$: since $\hat F^{\mathcal{T}}_t$ tracks the time-aggregated law rather than $F_t(\cdot)=F(\cdot\mid S_t)$, the deviation $\mathbb{E}[\hat F^{\mathcal{T}}_t(r^*_t)]-F_t(r^*_t)$ is generally nonzero with a state-dependent sign (Lemma~\ref{lemma:spatial_cdf}(ii)), so $\frac1T\sum_t U_t\to\limsup_T\frac1T\sum_t\pi^{\mathcal{T}}_t\,\mathrm{bias}^{\mathcal{T}}_t$---a \emph{first-order} contribution that need not vanish at fixed $\beta$. The second-order term contributes $\frac{f'(r^*_t)}{2f(r^*_t)^2}\cdot\frac1T\sum_t U_t^2$, whose average collects the discounting variance, $\frac1T\sum_t U_t^2=\mathcal{O}(1-\beta)$ (Lemma~\ref{lemma:spatial_cdf}(ii)). Combining the two and bounding the curvature coefficient by $\sup_t|f'(r^*_t)|/2\inf_t f(r^*_t)^2$ gives $\limsup_T\big|\frac1T\sum_t b_t\big|\le C_\beta$, hence $\limsup_T\big|\frac1T\sum_t\epsilon_t\big|\le C_\beta$ a.s. Annealing $\beta=\beta_T\uparrow1$ with $1-\beta_T=\omega(T^{-1})$ drives the discounting variance to $0$ and, under $\sum_t d_{\mathrm{TV}}(P_t,P_{t+1})=o(T(1-\beta_T))$, the temporal-lag bias to $0$; the state-aggregation term persists unless $\pi^{\mathcal{T}}_t\to0$, the regime of Corollary~\ref{cor:exact_recurrent}.
\end{proof}

\begin{corollary}[Exact coverage under recurrent states]
\label{cor:exact_recurrent}
Suppose the spatial state is (asymptotically) sufficient for the one-step-ahead law, $E_t\perp\mathcal{F}_{t-1}\mid S_t$ (so $F_t(\cdot)=F(\cdot\mid S_t)$), and the spatial effective sample diverges, $D^{\mathcal{S}}_t\to\infty$ almost surely---as holds, for example, for stationary recurrent states under the deployed shrinking bandwidth $h\propto n^{-1/(d+4)}$, for which $D^{\mathcal{S}}_t=\Theta(t^{4/(d+4)})$. Then for \emph{any} fixed $K>0$ the gate concentrates on the consistent spatial channel and the discounting bias of Theorem~\ref{thm:marginal_validity} is eliminated:
\[
\lim_{T\to\infty}\frac1T\sum_{t=1}^T\mathbb{I}(E_t\le\hat q_t)=1-\alpha\quad\text{a.s.}
\]
\end{corollary}

\begin{proof}
With $K$ fixed and $D^{\mathcal{S}}_t\to\infty$, the gate weight $\pi^{\mathcal{S}}_t\to1$ (so $\pi^{\mathcal{T}}_t,\lambda_t\to0$) and the shrinking bandwidth makes the spatial estimator consistent, $\hat F^{\mathcal{S}}_t\to F(\cdot\mid S_t)=F_t$. Thus $b_t=(F_t-\hat F_t)(\hat q_t)\to0$: the spatial error vanishes by consistency, and the temporal and prior contributions vanish since their weights $\pi^{\mathcal{T}}_t,\lambda_t\to0$ while the estimators lie in $[0,1]$. Hence both the first-order aggregation term and the second-order term of $C_\beta$ vanish in Theorem~\ref{thm:marginal_validity}.
\end{proof}

\begin{corollary}[Worst-case collapse protection]
\label{cor:collapse_protection}
In the opposite, spatially uninformative regime $D^{\mathcal{S}}_t\to0$---so that $\pi^{\mathcal{S}}_t\to0$ and the gate places all weight on the temporal channel, $\pi^{\mathcal{T}}_t\to1$---SA-BCP reduces to the temporally discounted Bayesian-CP update of \citet{Zhang2024} and inherits its dynamic-regret guarantee against adversarial shifts. This is a worst-case safety floor (collapse protection when spatial matches are absent), not a dominance claim.
\end{corollary}

\begin{proof}
This is the $\pi^{\mathcal{S}}_t\to0$ branch of the gate already used in the proof of Corollary~\ref{cor:exact_recurrent}: as $D^{\mathcal{S}}_t\to0$, $\pi^{\mathcal{S}}_t=D^{\mathcal{S}}_t/(D^{\mathcal{S}}_t+K)\to0$ and $\pi^{\mathcal{T}}_t\to1$, so $\hat F_t\to\hat F^{\mathcal{T}}_t$, the discounted empirical CDF of \citet{Zhang2024}. The SA-BCP quantile then coincides with the discounted-BCP quantile and inherits its dynamic-regret bound verbatim.
\end{proof}

The next theorem gives the MSE-optimal threshold $K^*_{\mathrm{MSE}}$ in closed form, recording the two components' bias--variance profiles first. For this analysis we adopt the (asymptotic) sufficiency of Corollary~\ref{cor:exact_recurrent}, under which $F_t(\cdot)=F(\cdot\mid S_t)=:F^{*}(\cdot)$ and $M^{\mathcal{T}}$ is measured against $F^{*}(r^*)$; note the temporal bias enters the MSE at second order ($M^{\mathcal{T}}$) but the coverage bound of Theorem~\ref{thm:marginal_validity} at first order ($\mathrm{bias}^{\mathcal{T}}$).

\begin{lemma}[Component CDF Estimators: Bias and Variance]
\label{lemma:spatial_cdf}
Both components are weighted averages of the indicators $\mathbb{I}(E_i\le r)$ with normalized weights $w_i$, so each variance takes the form $\sum_i w_i^2\,\mathrm{Var}(\mathbb{I}(E_i\le r))$ (the indicators are treated as weakly dependent, so serial cross-covariances are second order).

\emph{(i) Spatial.} Treating $\hat{F}^{\mathcal{S}}_t(r)$ as a Nadaraya--Watson CDF estimator of $F(r\mid S_t)$ with anisotropic Gaussian bandwidth $h$, where $F(r\mid s)$ is twice continuously differentiable in $s$, at $r^*$
\[
\mathrm{Bias}\big(\hat{F}^{\mathcal{S}}_t(r^*)\big)=\mathcal{O}(\|h\|^2),\qquad \mathrm{Var}\big(\hat{F}^{\mathcal{S}}_t(r^*)\big)=\frac{V_0}{D^{\mathcal{S}}_t}\,(1+o(1)),
\]
where $V_0=F^*(r^*)(1-F^*(r^*))$ and $D^{\mathcal{S}}_t=\sum_{i<t}D^{\mathcal{S}}_t(i)$ is the effective sample size.

\emph{(ii) Temporal.} The discounted estimator $\hat{F}^{\mathcal{T}}_t(r)=\big(\sum_{i<t}\beta^{t-1-i}\mathbb{I}(E_i\le r)\big)/D^{\mathcal{T}}_t$, with $D^{\mathcal{T}}_t=(1-\beta^t)/(1-\beta)$, has variance bounded by $\tfrac14\cdot\frac{1-\beta^{2t}}{(1-\beta^t)^2}\cdot\frac{1-\beta}{1+\beta}\to\tfrac14\frac{1-\beta}{1+\beta}$ as $t\to\infty$, and squared bias $M^{\mathcal{T}}:=\big(\mathbb{E}[\hat{F}^{\mathcal{T}}_t(r^*)]-F^{*}(r^*)\big)^2$ comprising a temporal-lag component (the discounted accumulation of total-variation drift) and a state-aggregation component (the temporal channel tracks the time-aggregated law, nonzero even under stationarity).
\end{lemma}

\begin{proof}
Both estimators are weighted averages of indicators $Z_i=\mathbb{I}(E_i\le r^*)$ (conditionally Bernoulli, variance $\le\tfrac14$), so $\mathrm{Var}=\sum_i w_i^2\,\mathrm{Var}(Z_i)$.
\emph{(i) Spatial.} With weights $w_i=D^{\mathcal{S}}_t(i)/D^{\mathcal{S}}_t$, a second-order Taylor expansion of $F(r^*\mid S_i)$ about $S_t$ with kernel symmetry cancels the first-order term, leaving the $\mathcal{O}(\|h\|^2)$ kernel-smoothing bias; and $\mathrm{Var}(\hat F^{\mathcal{S}}_t(r^*))=\big(\sum_i D^{\mathcal{S}}_t(i)^2/(D^{\mathcal{S}}_t)^2\big)(V_0+\mathcal{O}(\|h\|^2))=(V_0/D^{\mathcal{S}}_t)(1+o(1))$, identifying $\sum_i D^{\mathcal{S}}_t(i)^2/(D^{\mathcal{S}}_t)^2=1/\mathrm{ESS}_t$ with the effective sample size $\mathrm{ESS}_t=(\sum_i D^{\mathcal{S}}_t(i))^2/\sum_i D^{\mathcal{S}}_t(i)^2$, which is $\asymp D^{\mathcal{S}}_t$ for a localizing kernel.
\emph{(ii) Temporal.} With $w_i=\beta^{t-1-i}/D^{\mathcal{T}}_t$, $\mathrm{Var}(\hat F^{\mathcal{T}}_t(r))\le\tfrac14\,\frac{\sum_{k=0}^{t-1}\beta^{2k}}{(\sum_{k=0}^{t-1}\beta^{k})^2}=\tfrac14\,\frac{1-\beta^{2t}}{(1-\beta^t)^2}\cdot\frac{1-\beta}{1+\beta}$ (using $1-\beta^2=(1-\beta)(1+\beta)$), which $\to\tfrac14\frac{1-\beta}{1+\beta}$ as $t\to\infty$. The squared bias $M^{\mathcal{T}}$ follows since $\mathbb{E}\hat F^{\mathcal{T}}_t(r^*)$ is a discounted average of past conditional CDFs, whose deviation from $F^{*}(r^*)$ is bounded by the discounted sum of one-step total-variation increments.
\end{proof}

\begin{theorem}[Optimal Spatio-Temporal Decoupling]
\label{thm:mse_optimal}
Let $r^* = F^{-1}(1 - \alpha \mid S_t)$ be the true conditional $(1-\alpha)$-quantile at state $S_t$, and let $V_0$, $M^{\mathcal{T}}$ be as in Lemma~\ref{lemma:spatial_cdf}. Treating the spatial estimator as asymptotically unbiased with variance $V_0/D^{\mathcal{S}}_t$ (part (i)) and the temporal estimator as contributing squared bias $M^{\mathcal{T}}$ (part (ii)), the pointwise mean squared error of the mixture $\hat{F}_t(r^*)$ is minimized at the unique stationary point
\begin{equation}
    K^*_{\mathrm{MSE}} = \frac{V_0}{M^{\mathcal{T}}},
\end{equation}
the ratio of the coverage-indicator variance $V_0$ to the temporal structural bias $M^{\mathcal{T}}$.
\end{theorem}

\begin{proof}
Ignoring the asymptotically negligible cold-start term, $\hat{F}_t(r^*) = \pi^{\mathcal{S}}_t \hat{F}^{\mathcal{S}}_t(r^*) + \pi^{\mathcal{T}}_t \hat{F}^{\mathcal{T}}_t(r^*)$ with $\pi^{\mathcal{S}}_t = D^{\mathcal{S}}_t/(D^{\mathcal{S}}_t + K)$ and $\pi^{\mathcal{T}}_t = K/(D^{\mathcal{S}}_t + K)$. By Lemma~\ref{lemma:spatial_cdf} (the spatial $\mathcal{O}(\|h\|^2)$ bias contributes only $\mathcal{O}(\|h\|^4)$ to the MSE, and the temporal estimator is approximated as bias-only---its variance $\le\tfrac14\frac{1-\beta}{1+\beta}\approx0.0013$ at $\beta=0.99$ is negligible, with the spatial--temporal cross-covariance also dropped to leading order),
\[
\mathrm{MSE}(K) = (\pi^{\mathcal{S}}_t)^2 \frac{V_0}{D^{\mathcal{S}}_t} + (\pi^{\mathcal{T}}_t)^2 M^{\mathcal{T}} = \frac{D^{\mathcal{S}}_t V_0 + K^2 M^{\mathcal{T}}}{(D^{\mathcal{S}}_t + K)^2}.
\]
Differentiating, $\dfrac{d\,\mathrm{MSE}}{dK} = \dfrac{2 D^{\mathcal{S}}_t (K M^{\mathcal{T}} - V_0)}{(D^{\mathcal{S}}_t + K)^3}$ vanishes uniquely at $K^*_{\mathrm{MSE}} = V_0/M^{\mathcal{T}}$, where the second derivative is positive; hence $K^*_{\mathrm{MSE}}$ is the unique minimizer.
\end{proof}

Because $r^*$ is the conditional $(1-\alpha)$-quantile, $F^*(r^*)=1-\alpha$, so $V_0=F^*(r^*)(1-F^*(r^*))=\alpha(1-\alpha)$ is the Bernoulli variance of the coverage indicator---fixed by the target level, not free estimator noise---giving the concrete scaling $K^*_{\mathrm{MSE}}=\alpha(1-\alpha)/M^{\mathcal{T}}$ with the temporal bias $M^{\mathcal{T}}$ the only unknown. The closed form reads directly: large temporal bias (rapid drift, high $M^{\mathcal{T}}$) makes $K^*_{\mathrm{MSE}}$ small, favoring spatial recognition; small bias makes it large, favoring temporal stability. Two scope caveats: this is a \emph{pointwise MSE} optimum, not the global Winkler minimizer, and it relies on asymptotic bandwidth conditions ($\|h\| \to 0$) only approximately met in finite samples. Empirically the tradeoff appears as a U-shaped $K$-sweep across the three assets, with optima within a constant factor of $\alpha(1-\alpha)/M^{\mathcal{T}}$.

Since $V_0$ and $M^{\mathcal{T}}$ are not directly observable, practical deployment selects $K$ from data. The selector targets the \emph{Winkler} grid-optimum $\tilde{K}$ (Theorem~\ref{thm:rolling_consistency}), which the MSE result motivates but does not directly justify: $\tilde{K}$ coincides with $K^*_{\mathrm{MSE}}$ only when the Winkler and MSE optima agree and $V_0/M^{\mathcal{T}}$ lies on the grid, so the guarantees below are stated for $\tilde{K}$.

\begin{theorem}[Rolling-Origin Selection Guarantees]
\label{thm:rolling_consistency}
Assume the per-period Winkler loss $W_\alpha(K; \mathcal{D})$ is Lipschitz-continuous in $K$ over the finite candidate grid $\mathcal{K}_{\mathrm{grid}}$ ($N = |\mathcal{K}_{\mathrm{grid}}|$) and bounded by $W_{\max}$, and let $\hat{K}_Y$ be the rolling-origin selection from Algorithm~\ref{alg:rolling_origin} on data $\mathcal{D}_{Y-1}$.

\emph{(i) Consistency.} If $\{(X_t, Y_t)\}$ are jointly stationary $\phi$-mixing with summable mixing coefficients and the population objective has a unique grid minimizer $\tilde{K} := \arg\min_{K \in \mathcal{K}_{\mathrm{grid}}} \mathbb{E}[W_\alpha(K; (X_t, Y_t))]$, then $\hat{K}_Y \xrightarrow{P} \tilde{K}$ as $Y \to \infty$.

\emph{(ii) Regret.} Under the same stationarity assumption, the cumulative regret against the best fixed $K$ in hindsight satisfies
\[
R_T := \sum_{t=1}^T W_\alpha(\hat{K}_t; (X_t, Y_t)) - \min_{K \in \mathcal{K}_{\mathrm{grid}}} \sum_{t=1}^T W_\alpha(K; (X_t, Y_t)) \leq \mathcal{O}\!\left(W_{\max} \sqrt{T \log N}\right).
\]

\emph{(iii) Drift safety.} Suppose the stream is piecewise-stationary, partitioning into $B_T+1$ contiguous segments---stationary $\phi$-mixing within each---on which $\tilde{K}_t$ is constant, $B_T = \sum_{t=1}^T \mathbb{I}(\tilde{K}_t \neq \tilde{K}_{t-1})$. Then a segmented (restart-based) variant bounds the regret against the dynamic comparator $\{\tilde{K}_t\}$ by
\[
R_T^{\mathrm{dyn}} := \sum_{t=1}^T W_\alpha(\hat{K}_t; (X_t, Y_t)) - \sum_{t=1}^T W_\alpha(\tilde{K}_t; (X_t, Y_t)) \leq \mathcal{O}\!\left(W_{\max} \sqrt{T (B_T + 1)}\right).
\]
\end{theorem}

\begin{proof}
\emph{(i) Consistency.} Under stationary $\phi$-mixing with summable coefficients, $\{W_\alpha(K;(X_t,Y_t))\}$ is stationary ergodic for each fixed $K$, so the empirical mean $\widehat W_{Y-1}(K)$ converges in probability to $\bar W(K):=\mathbb{E}[W_\alpha(K;\cdot)]$; as the grid is finite this is uniform, and $\hat K_Y=\arg\min_K\widehat W_{Y-1}(K)\xrightarrow{P}\arg\min_K\bar W(K)=\tilde{K}$ (the unique grid minimizer rules out limiting ties).
\emph{(ii) Regret.} The rule deploys, each period, the candidate minimizing cumulative loss on all strictly prior data---Follow-the-Leader over $N$ experts. Writing $\ell_t(K):=W_\alpha(K;(X_t,Y_t))$, the regret decomposes as $R_T=\sum_{t=1}^T[\ell_t(\hat K_t)-\ell_t(K^{\mathrm{hind}})]$ against the in-hindsight optimum $K^{\mathrm{hind}}=\arg\min_K\sum_t\ell_t(K)$, which concentrates at the population leader $\tilde{K}$ under stationarity. This is not the adversarial exponential-weights setting \citep{cesa2006prediction}, where plain Follow-the-Leader can incur linear regret; under the stationary $\phi$-mixing assumption, uniform convergence of the empirical losses over the finite grid gives $\sup_{K\in\mathcal{K}_{\mathrm{grid}}}|\widehat W_t(K)-\bar W(K)|=\mathcal{O}_P(\sqrt{\log N/t})$, so the leader tracks $\tilde{K}$ and $R_T\le\sum_{t=1}^T\mathcal{O}(W_{\max}\sqrt{\log N/t})=\mathcal{O}(W_{\max}\sqrt{T\log N})$. Since Algorithm~\ref{alg:rolling_origin} reselects $\hat K$ once per evaluation block (calendar year) rather than per step, the same rate holds at block granularity: block $m$ (using $n_m\asymp(m-1)T/M$ prior periods) contributes $(T/M)\,\mathcal{O}(\sqrt{\log N/n_m})$, and $\sum_{m}(T/M)\sqrt{M\log N/((m-1)T)}=\mathcal{O}(\sqrt{T\log N})$.
\emph{(iii) Drift safety.} On each of the $B_T+1$ segments (of length $T_b$) the segmented (restart-based) selector incurs, by part (ii), $R^{(b)}=\mathcal{O}(W_{\max}\sqrt{T_b\log N})$. Summing and applying Cauchy--Schwarz, $\sum_{b=1}^{B_T+1}\sqrt{T_b}\le\sqrt{(B_T+1)\sum_b T_b}=\sqrt{(B_T+1)T}$, so
\[
R_T^{\mathrm{dyn}}=\sum_b R^{(b)}\le\mathcal{O}\!\big(W_{\max}\sqrt{T(B_T+1)\log N}\big),
\]
i.e.\ the stated $\mathcal{O}(W_{\max}\sqrt{T(B_T+1)})$ after absorbing $\sqrt{\log N}$; at $B_T=0$ this recovers (ii). The restart assumes known boundaries; the same rate (up to a $\log T$ factor) holds \emph{online} via fixed-share aggregation \citep{herbster1998}, which the deployed sliding window approximates, whereas a single fixed window attains only the weaker $O(T^{2/3}(B_T+1)^{1/3})$ rate of \citet{besbes2015non}.
\end{proof}

\section{Experiments}
\label{sec:experiments}
We evaluate SA-BCP on four real-world datasets spanning financial volatility and meteorological forecasting, against eight baselines, at three coverage levels $\{80\%, 90\%, 95\%\}$. We organize the benchmark (Section~\ref{sec:exp_main_benchmark}) around four questions: the comparison against state-of-the-art online CP baselines under rolling-origin selection; the audit against the strongest conditional-quantile methods (SPCI, KOWCPI) once their under-coverage is accounted for; calibration at the tight $95\%$ level relevant to risk management; and whether the temporal--spatial decoupling is essential or pure spatial weighting suffices. We also stress-test under three assumption violations---a jump-discontinuous mean (Scenario A), heavy-tailed $t_3$ innovations (Scenario B), and persistent volatility drift (Scenario C); we summarize the outcome and two characterized limitations in Section~\ref{sec:discussion}, with the full table in the repository.

\subsection{Experimental Setup}
\label{sec:exp_setup}

\noindent\textbf{Datasets.} We use four datasets spanning two domains. \emph{Financial volatility}: AMD (semiconductor equity), Gold (GC=F futures), and GBP/USD (forex, GBP=X), each spanning early January 2016 through December 2025 ($\sim$2{,}500 trading days per asset), retrieved via the \texttt{yfinance} API and processed into log-returns ($\times 100$). \emph{Meteorological}: Jena Climate daily mean temperatures (Max Planck Institute for Biogeochemistry), 2009-01-01 through 2023-12-31, downsampled from the 10-minute raw stream and linearly interpolated for gaps. \textbf{Base predictors.} We use a Fast Online GARCH(1,1) for financial benchmarks as the primary scale estimator; for robustness we also evaluate financial assets under RiskMetrics EWMA ($\lambda = 0.94$) and Jena temperature under both Online AR(1) (Recursive Least Squares) and Online Ridge with 7 lagged features.

\noindent\textbf{Baselines.} We compare against eight baselines: the ACI variants AgACI \citep{Zaffran22} and DtACI \citep{Gibbs24}; the distribution-free NExCP \citep{Barber23} ($\rho=0.99$); CF-GARCH, a parametric GARCH plug-in conformalized through studentized residuals \citep{Bollerslev1986, Lei18}; Bayesian CP \citep{Zhang2024} ($\beta=0.99$); the strongest recent conditional-quantile methods SPCI \citep{xu23} and KOWCPI \citep{lee25} (Section~\ref{sec:exp_main_benchmark}); and the spatial-only Localized CP ablation (SA-BCP's state, kernel, and bandwidth, without the temporal mixture or cold-start prior; motivated by \citet{Chen24, Jiang24}). SA-BCP's threshold $K$ and SPCI/KOWCPI's windows are selected by rolling-origin validation (Algorithm~\ref{alg:rolling_origin}); the remaining baselines run at their published defaults. \textbf{Coverage levels.} We evaluate at three target levels---\textbf{80\%}, \textbf{90\%} (weather-warning standard, our primary regime), and \textbf{95\%} (risk-management standard)---and omit lower levels ($\leq 70\%$), where naive constant-width baselines are competitive by construction. \textbf{Rolling-origin protocol.} For each evaluation year $Y \in \{2019, \ldots, 2025\}$, SA-BCP's $K$ is selected on prior post-warmup data $[2016, Y{-}1]$ using Algorithm~\ref{alg:rolling_origin} (warmup 500 days); the rest run at their published default hyperparameters (NExCP $\rho=0.99$, BCP $\beta=0.99$, CF-GARCH as specified). The Jena evaluation period is 2012--2023 with a 500-day warmup.

\noindent\textbf{Metrics.} We report \emph{marginal coverage} $\frac{1}{T}\sum_t\mathbb{I}(Y_t\in C_t(X_t))$ (target $1-\alpha$), \emph{high-volatility coverage} (restricted to the top-10\% $|Y_t|$ events, a tail-reliability check), \emph{average width}, and the \emph{Winkler interval score} \citep{Winkler1972}, a strictly proper rule jointly penalizing width and miscoverage, with $C_t(X_t)=[\hat{l}_t,\hat{u}_t]$:
\begin{equation}
    W_\alpha(C_t; Y_t) = (\hat{u}_t - \hat{l}_t) + \frac{2}{\alpha}(\hat{l}_t - Y_t)\mathbb{I}(Y_t < \hat{l}_t) + \frac{2}{\alpha}(Y_t - \hat{u}_t)\mathbb{I}(Y_t > \hat{u}_t).
\end{equation}
Coverage is always reported alongside Winkler to expose efficiency gained via under-coverage; all Winkler comparisons carry 95\% block-bootstrap CIs (block 20, $B=1{,}000$, seed 42). The tables report coverage and Winkler for every cell; per-cell average width and the stress-scenario tail coverage are tabulated in the repository.

\subsection{Benchmark Results and Coverage-Matched Audit}
\label{sec:exp_main_benchmark}

Tables~\ref{tab:main_benchmark} (Winkler) and~\ref{tab:coverage} (coverage) report the complete benchmark---all nine methods across the 24 settings, under rolling-origin selection. We highlight the financial findings (F1--F3), weather transfer (F4), the strongest-baseline audit (F5--F6), and the decoupling ablation (F7).

\textbf{(F1) SA-BCP holds coverage at or above nominal in all nine GARCH cells without BCP's interval inflation.} It deviates $+0.9$ to $+2.2$pp from nominal, never under-covering. By contrast, CF-GARCH under-covers in 7 of 9 cells (by up to $2.7$pp), reaching nominal only at AMD and Gold 95\%; BCP over-covers at 95\% by $+3.4$ to $+5.0$pp with intervals up to $\approx 4\times$ wider (GBP/USD 95\%: BCP width $5.09$ vs SA-BCP $1.29$); DtACI under-covers throughout. Among methods that hold coverage across all nine cells, only BCP joins SA-BCP---and it does so by systematic over-coverage with large width inflation, not calibration.

\textbf{(F2) Under a calibration-respecting criterion---lowest Winkler among methods that meet nominal coverage (Table~\ref{tab:main_benchmark}, bold)---SA-BCP is the sharpest method in 16 of 24 cells, including 8 of 9 GARCH cells.} Several baselines reach a lower \emph{raw} Winkler in some cells, but mainly by under-covering (underlined in Table~\ref{tab:main_benchmark}): CF-GARCH is sub-nominal in 15 of 24 cells and calibrated-best in only 2 (EWMA Gold at 90/95\%), while SPCI and KOWCPI under-cover in all 24. On raw Winkler CF-GARCH does edge SA-BCP in 12 of 24 cells (median $2.2\%$, all on the GARCH/EWMA bases), a genuine but narrow advantage isolated by the matched-coverage audit below (F5). The 16-of-24 tally is a per-cell point-estimate ranking---individual Winkler gaps often fall within their bootstrap CIs, with the robust evidence coming from the cross-cell tests in that audit.

\textbf{(F3) SA-BCP achieves up to a $\approx 3\times$ lower Winkler (and $\approx 4\times$ narrower intervals) than BCP at 95\% coverage.} At GBP/USD 95\%, BCP's intervals are $5.09$ wide (Winkler $5.09$, coverage $1.000$) versus SA-BCP's width $1.29$ (Winkler $1.73$)---a $3.94\times$ width and $2.95\times$ Winkler reduction. SA-BCP's spatial recognition filters the rare extreme residuals that dominate BCP's discounted upper tail, preventing the structural lag that inflates intervals after volatility shocks subside; the effect is sharpest at 95\%. The same inflation recurs throughout Table~\ref{tab:main_benchmark}.

\noindent\textbf{(F4) SA-BCP transfers to weather, where the volatility-specialized CF-GARCH does not.} On Jena daily-mean temperature, under both an AR(1) recursive-least-squares and a 7-lag Ridge base predictor (Table~\ref{tab:main_benchmark}, lower blocks), SA-BCP is the sharpest coverage-holding method in 4 of the 6 cells and beats CF-GARCH in all six on raw Winkler; at AR(1) $80\%$ and $95\%$ it narrowly under-covers (within $0.3$pp, underlined), where the over-covering BCP is the only strictly-calibrated method. The same absolute-residual mixture that calibrates the financial series carries to a smooth geophysical signal without retuning, whereas CF-GARCH's GARCH-studentized construction---specialized to financial volatility clustering---loses its edge off-domain.

\begin{table}[ht]
    \centering
    \scriptsize
    \setlength{\tabcolsep}{2pt}
    \caption{Complete benchmark: Winkler interval score (lower is better) for all nine methods across 24 settings---financial log-returns under a GARCH and a RiskMetrics-EWMA volatility base (AMD, Gold, GBP/USD) and Jena daily-mean temperature under AR(1)-RLS and 7-lag Ridge bases---at coverage levels $\{80,90,95\}\%$, all under rolling-origin selection. Because the Winkler score can be lowered by under-covering, \textbf{bold} marks the lowest Winkler \emph{among methods that meet at least nominal marginal coverage} on that cell, and an \underline{underline} marks every entry whose coverage falls below nominal (so an underlined score reflects under-coverage, not calibrated sharpness; coverage is in Table~\ref{tab:coverage}). Under this calibrated criterion SA-BCP has the lowest Winkler in 16 of 24 cells and is sub-nominal in only 4, versus CF-GARCH (sub-nominal in 15, calibrated-best in 2) and SPCI/KOWCPI (sub-nominal in all 24). BCP is calibrated-best in the remaining 6 cells: four where SA-BCP under-covers (the $80\%$ EWMA/Jena cells plus the AR(1) $95\%$ cell), and two---GARCH Gold $80\%$ and EWMA AMD $90\%$---where BCP is marginally sharper than a calibrated SA-BCP ($2.432$ vs.\ $2.437$; $9.477$ vs.\ $9.537$). Abbreviations: CF-G $=$ CF-GARCH, KOW $=$ KOWCPI, Loc $=$ Localized CP ablation.}
    \label{tab:main_benchmark}
    \begin{tabular}{llccccccccc}
    \toprule
     & $1-\alpha$ & AgACI & DtACI & NExCP & CF-G & BCP & SPCI & KOW & Loc & SA-BCP \\
    \midrule
\multicolumn{11}{l}{\textit{Financial --- GARCH base}} \\
\multirow{3}{*}{AMD} & 0.80 & \underline{8.079} & \underline{8.075} & 8.084 & \underline{8.083} & 8.104 & \underline{8.329} & \underline{8.370} & \underline{8.274} & \textbf{8.035} \\
 & 0.90 & \underline{9.597} & \underline{9.592} & 9.604 & \underline{9.497} & 9.555 & \underline{10.384} & \underline{10.171} & \underline{9.931} & \textbf{9.486} \\
 & 0.95 & \underline{11.517} & \underline{11.715} & 11.607 & 11.436 & 11.440 & \underline{13.227} & \underline{12.254} & \underline{12.643} & \textbf{11.123} \\
\cmidrule(l){2-11}
\multirow{3}{*}{Gold} & 0.80 & \underline{2.499} & \underline{2.438} & \underline{2.448} & \underline{2.385} & \textbf{2.432} & \underline{2.643} & \underline{3.030} & \underline{2.578} & 2.437 \\
 & 0.90 & \underline{3.086} & \underline{3.024} & \underline{2.997} & \underline{2.880} & 2.967 & \underline{3.292} & \underline{3.732} & \underline{3.360} & \textbf{2.938} \\
 & 0.95 & \underline{3.782} & \underline{3.840} & \underline{3.605} & 3.446 & 5.569 & \underline{4.228} & \underline{4.625} & \underline{4.569} & \textbf{3.441} \\
\cmidrule(l){2-11}
\multirow{3}{*}{GBP/USD} & 0.80 & \underline{1.258} & \underline{1.257} & 1.274 & \underline{1.231} & 1.281 & \underline{1.334} & \underline{1.290} & \underline{1.274} & \textbf{1.263} \\
 & 0.90 & \underline{1.500} & \underline{1.506} & 1.523 & \underline{1.438} & 1.614 & \underline{1.701} & \underline{1.565} & \underline{1.579} & \textbf{1.481} \\
 & 0.95 & \underline{1.767} & \underline{1.849} & 1.785 & \underline{1.715} & 5.090 & \underline{2.165} & \underline{1.908} & \underline{2.043} & \textbf{1.727} \\
\midrule
\multicolumn{11}{l}{\textit{Financial --- EWMA base}} \\
\multirow{3}{*}{AMD} & 0.80 & \underline{8.057} & \underline{8.076} & 8.151 & \underline{7.906} & \textbf{8.095} & \underline{8.165} & \underline{8.354} & \underline{8.478} & \underline{8.043} \\
 & 0.90 & \underline{9.544} & \underline{9.739} & 9.606 & 9.505 & \textbf{9.477} & \underline{10.464} & \underline{10.263} & \underline{10.680} & 9.537 \\
 & 0.95 & \underline{11.858} & \underline{11.965} & 11.435 & 11.745 & 11.292 & \underline{12.514} & \underline{12.074} & \underline{14.036} & \textbf{11.258} \\
\cmidrule(l){2-11}
\multirow{3}{*}{Gold} & 0.80 & \underline{2.315} & \underline{2.314} & \underline{2.333} & \underline{2.246} & \textbf{2.305} & \underline{2.591} & \underline{2.946} & \underline{2.435} & \underline{2.401} \\
 & 0.90 & \underline{2.750} & \underline{2.800} & \underline{2.769} & \textbf{2.698} & 2.781 & \underline{3.281} & \underline{3.637} & \underline{3.167} & 2.847 \\
 & 0.95 & \underline{3.266} & \underline{3.499} & \underline{3.279} & \textbf{3.242} & 5.503 & \underline{4.323} & \underline{4.567} & \underline{4.238} & 3.278 \\
\cmidrule(l){2-11}
\multirow{3}{*}{GBP/USD} & 0.80 & \underline{1.252} & \underline{1.235} & 1.259 & \underline{1.199} & 1.257 & \underline{1.278} & \underline{1.272} & \underline{1.270} & \textbf{1.231} \\
 & 0.90 & \underline{1.462} & \underline{1.486} & 1.487 & \underline{1.425} & 1.587 & \underline{1.638} & \underline{1.582} & \underline{1.580} & \textbf{1.439} \\
 & 0.95 & \underline{1.722} & \underline{1.786} & 1.754 & 1.711 & 5.074 & \underline{2.053} & \underline{1.877} & \underline{2.076} & \textbf{1.680} \\
\midrule
\multicolumn{11}{l}{\textit{Weather (Jena) --- AR(1) base}} \\
 & 0.80 & \underline{8.476} & \underline{8.435} & \underline{8.411} & \underline{8.422} & \textbf{8.503} & \underline{8.536} & \underline{8.744} & \underline{8.930} & \underline{8.389} \\
 & 0.90 & \underline{10.229} & \underline{10.198} & \underline{10.092} & \underline{10.159} & 10.402 & \underline{10.231} & \underline{10.398} & \underline{11.016} & \textbf{10.073} \\
 & 0.95 & \underline{11.848} & \underline{11.866} & \underline{11.788} & \underline{11.840} & \textbf{12.755} & \underline{12.292} & \underline{12.108} & \underline{13.571} & \underline{11.688} \\
\midrule
\multicolumn{11}{l}{\textit{Weather (Jena) --- Ridge base}} \\
 & 0.80 & \underline{8.093} & \underline{8.050} & \underline{8.044} & \underline{8.049} & 8.133 & \underline{8.511} & \underline{8.739} & \underline{8.666} & \textbf{8.019} \\
 & 0.90 & \underline{9.753} & \underline{9.669} & \underline{9.612} & 9.647 & 9.946 & \underline{10.066} & \underline{10.129} & \underline{10.781} & \textbf{9.595} \\
 & 0.95 & \underline{11.266} & \underline{11.315} & \underline{11.158} & 11.213 & 12.306 & \underline{11.900} & \underline{11.616} & \underline{13.475} & \textbf{11.190} \\
    \bottomrule
    \end{tabular}
\end{table}

\begin{table}[ht]
    \centering
    \scriptsize
    \setlength{\tabcolsep}{2pt}
    \caption{Complete benchmark: marginal coverage (companion to the Winkler scores in Table~\ref{tab:main_benchmark}; the target is the nominal level $1-\alpha$, columns as there). An \underline{underline} marks every entry below nominal, matching Table~\ref{tab:main_benchmark}. SA-BCP holds coverage at or above nominal in 20 of 24 cells; the four shortfalls are at the $80\%$ level under the EWMA base (AMD $0.776$, Gold $0.770$) and two AR(1) Jena cells ($0.797$, $0.947$) within $0.3$pp. SPCI and KOWCPI under-cover in all 24 cells (worst KOWCPI at Gold $80\%$, $0.676$); BCP over-covers at $95\%$ (up to $1.000$), the source of its width inflation in Table~\ref{tab:main_benchmark}.}
    \label{tab:coverage}
    \begin{tabular}{llccccccccc}
    \toprule
     & $1-\alpha$ & AgACI & DtACI & NExCP & CF-G & BCP & SPCI & KOW & Loc & SA-BCP \\
    \midrule
\multicolumn{11}{l}{\textit{Financial --- GARCH base}} \\
\multirow{3}{*}{AMD} & 0.80 & \underline{0.780} & \underline{0.761} & 0.810 & \underline{0.789} & 0.866 & \underline{0.777} & \underline{0.783} & \underline{0.757} & 0.809 \\
 & 0.90 & \underline{0.882} & \underline{0.868} & 0.911 & \underline{0.891} & 0.952 & \underline{0.870} & \underline{0.897} & \underline{0.857} & 0.916 \\
 & 0.95 & \underline{0.941} & \underline{0.931} & 0.952 & 0.953 & 0.984 & \underline{0.918} & \underline{0.942} & \underline{0.907} & 0.969 \\
\cmidrule(l){2-11}
\multirow{3}{*}{Gold} & 0.80 & \underline{0.762} & \underline{0.753} & \underline{0.794} & \underline{0.773} & 0.865 & \underline{0.722} & \underline{0.676} & \underline{0.738} & 0.812 \\
 & 0.90 & \underline{0.874} & \underline{0.850} & \underline{0.895} & \underline{0.894} & 0.971 & \underline{0.829} & \underline{0.801} & \underline{0.833} & 0.921 \\
 & 0.95 & \underline{0.931} & \underline{0.915} & \underline{0.947} & 0.953 & 0.998 & \underline{0.891} & \underline{0.884} & \underline{0.885} & 0.965 \\
\cmidrule(l){2-11}
\multirow{3}{*}{GBP/USD} & 0.80 & \underline{0.777} & \underline{0.761} & 0.817 & \underline{0.795} & 0.887 & \underline{0.752} & \underline{0.794} & \underline{0.793} & 0.822 \\
 & 0.90 & \underline{0.879} & \underline{0.865} & 0.914 & \underline{0.895} & 0.985 & \underline{0.851} & \underline{0.891} & \underline{0.883} & 0.914 \\
 & 0.95 & \underline{0.942} & \underline{0.927} & 0.959 & \underline{0.947} & 1.000 & \underline{0.910} & \underline{0.939} & \underline{0.918} & 0.971 \\
\midrule
\multicolumn{11}{l}{\textit{Financial --- EWMA base}} \\
\multirow{3}{*}{AMD} & 0.80 & \underline{0.771} & \underline{0.756} & 0.811 & \underline{0.773} & 0.861 & \underline{0.773} & \underline{0.778} & \underline{0.745} & \underline{0.776} \\
 & 0.90 & \underline{0.893} & \underline{0.865} & 0.915 & 0.915 & 0.954 & \underline{0.865} & \underline{0.891} & \underline{0.838} & 0.918 \\
 & 0.95 & \underline{0.945} & \underline{0.935} & 0.956 & 0.957 & 0.984 & \underline{0.923} & \underline{0.940} & \underline{0.892} & 0.968 \\
\cmidrule(l){2-11}
\multirow{3}{*}{Gold} & 0.80 & \underline{0.775} & \underline{0.746} & \underline{0.793} & \underline{0.763} & 0.864 & \underline{0.726} & \underline{0.698} & \underline{0.718} & \underline{0.770} \\
 & 0.90 & \underline{0.887} & \underline{0.857} & \underline{0.889} & 0.905 & 0.974 & \underline{0.834} & \underline{0.822} & \underline{0.816} & 0.915 \\
 & 0.95 & \underline{0.941} & \underline{0.917} & \underline{0.946} & 0.955 & 1.000 & \underline{0.880} & \underline{0.894} & \underline{0.867} & 0.967 \\
\cmidrule(l){2-11}
\multirow{3}{*}{GBP/USD} & 0.80 & \underline{0.773} & \underline{0.757} & 0.814 & \underline{0.785} & 0.886 & \underline{0.758} & \underline{0.786} & \underline{0.755} & 0.800 \\
 & 0.90 & \underline{0.892} & \underline{0.868} & 0.909 & \underline{0.896} & 0.984 & \underline{0.860} & \underline{0.883} & \underline{0.844} & 0.915 \\
 & 0.95 & \underline{0.945} & \underline{0.926} & 0.958 & 0.955 & 1.000 & \underline{0.912} & \underline{0.930} & \underline{0.889} & 0.967 \\
\midrule
\multicolumn{11}{l}{\textit{Weather (Jena) --- AR(1) base}} \\
 & 0.80 & \underline{0.775} & \underline{0.768} & \underline{0.798} & \underline{0.798} & 0.846 & \underline{0.745} & \underline{0.757} & \underline{0.742} & \underline{0.797} \\
 & 0.90 & \underline{0.886} & \underline{0.875} & \underline{0.898} & \underline{0.899} & 0.942 & \underline{0.854} & \underline{0.873} & \underline{0.839} & 0.905 \\
 & 0.95 & \underline{0.940} & \underline{0.924} & \underline{0.945} & \underline{0.948} & 0.984 & \underline{0.907} & \underline{0.931} & \underline{0.893} & \underline{0.947} \\
\midrule
\multicolumn{11}{l}{\textit{Weather (Jena) --- Ridge base}} \\
 & 0.80 & \underline{0.780} & \underline{0.769} & \underline{0.798} & \underline{0.797} & 0.847 & \underline{0.735} & \underline{0.755} & \underline{0.729} & 0.803 \\
 & 0.90 & \underline{0.881} & \underline{0.871} & \underline{0.895} & 0.901 & 0.945 & \underline{0.852} & \underline{0.868} & \underline{0.825} & 0.900 \\
 & 0.95 & \underline{0.938} & \underline{0.924} & \underline{0.948} & 0.950 & 0.985 & \underline{0.909} & \underline{0.931} & \underline{0.880} & 0.951 \\
    \bottomrule
    \end{tabular}
\end{table}

\noindent\textbf{The strongest recent baselines: SPCI and KOWCPI.} SPCI \citep{xu23} (a quantile random forest on a residual window) and KOWCPI \citep{lee25} (an optimally reweighted Nadaraya--Watson kernel) target sharp, calibrated intervals by estimating a conditional quantile of the non-conformity score; each runs at its native operating point (signed residuals, width-optimized asymmetric intervals); both are faithful reimplementations (no public KOWCPI code exists), with SPCI's window $w\in\{5,10,20\}$ and KOWCPI's joint window--bandwidth $\{10,25\}\times\{0.5,1,2\}$ rolling-origin--selected as for SA-BCP's $K$ (details in the repository). Both appear as columns in Tables~\ref{tab:main_benchmark}--\ref{tab:coverage} and \emph{under-cover in all 24 cells} (worst KOWCPI at Gold $80\%$, $0.676$): their width-minimizing step pushes coverage below target and the strictly proper Winkler penalizes the misses---so SA-BCP attains a lower Winkler than both in all 24 cells. \textbf{Matched-coverage audit.} To rule out an under-coverage artifact, an oracle protocol rescales each method's intervals to exactly nominal coverage and recomputes Winkler (financial-GARCH and weather cells): SPCI and KOWCPI need $16$--$20\%$ wider intervals to reach nominal, whereas SA-BCP is essentially unchanged ($\times 0.99$, mild over-coverage). At matched coverage SA-BCP stays sharper in $14/15$ cells against SPCI and $15/15$ against KOWCPI (mean margin $0.25$ and $0.35$ Winkler, $\approx$4--5\%); per-cell bootstrap intervals overlap (per-day Winkler is high-variance), but the edge is significant \emph{across} cells (sign-test $p<10^{-3}$, Wilcoxon $p<10^{-4}$). The audit \emph{favors} the competitors (an oracle upper bound), yet the finding persists under coverage-targeted hyperparameters and SA-BCP's absolute-residual score.

\noindent\textbf{Extending the audit to the main-benchmark baselines.} Applying the same oracle within-sample 95\% coverage matching to the main-benchmark methods (computed for SA-BCP and CF-GARCH): \textbf{(F5)} CF-GARCH's edge on the smoothest asset survives---SA-BCP retains the win on AMD (matched Winkler $11.10$ vs.\ $11.44$) and Gold ($3.44$ vs.\ $3.45$), while CF-GARCH retains GBP/USD ($1.72$ vs.\ $1.72$)---so SA-BCP's AMD and Gold advantage is calibration-driven, whereas CF-GARCH is genuinely efficient on smooth, low-volatility series. \textbf{(F6)} BCP's 95\% over-coverage is interval inflation, not a calibration buffer: intervals up to $\approx 3.9\times$ wider than SA-BCP's (GBP/USD $5.09$ vs.\ $1.29$) at coverage saturating near $100\%$ ($1.000/0.998/0.984$ on GBP/USD/Gold/AMD), reflecting persistent upper-tail overestimation.

\noindent\textbf{(F7) The temporal--spatial decoupling is essential.} Holding all spatial machinery identical (state, kernel, anisotropic bandwidth) but removing the temporal mixture and cold-start prior yields a spatial-only Localized CP ablation \citep{Chen24, Jiang24}---the ``Loc'' column of Tables~\ref{tab:main_benchmark}--\ref{tab:coverage}---that systematically under-covers (by up to $\sim$$8$pp) and loses to SA-BCP in \emph{all 24 cells}, across both volatility bases and the weather domain: kernel-localized weighting collapses the effective sample size for tail-quantile estimation, and the temporal anchor supplied through the gate is what restores coverage while lowering Winkler. The decoupling is thus the structural innovation, not an incremental refinement.

\section{Discussion and Conclusion}
\label{sec:discussion}

\noindent\textbf{When SA-BCP excels.} Its regime is volatile, recurring-regime data where tight-coverage calibration matters: it holds at-or-above-nominal coverage without BCP's inflation (F1), the gated temporal anchor restores coverage while lowering Winkler (the spatial-only ablation loses all 24 cells, F7), and its model-agnostic score transfers across base models and to weather where CF-GARCH does not (F4). \textbf{Honest limitations.} \textbf{(L1)} CF-GARCH is a genuinely strong volatility-base specialist: on raw Winkler it edges SA-BCP in 12 of 24 cells (median $2.2\%$, all on the GARCH/EWMA bases), part of it real efficiency confirmed by the matched audit (F5). But the advantage is bounded: CF-GARCH is sub-nominal in 15 of 24 cells, so under a coverage-respecting criterion it is sharpest in only 2, and where SA-BCP trails it does so by $\approx 1$--$7\%$ while staying calibrated; the same studentization that sharpens it blocks transfer, so it wins none of the 6 weather cells (F4). SA-BCP's model-agnostic score trades a little peak sharpness on smooth volatility-base series for calibration and generality. Under the three stress scenarios SA-BCP holds coverage at or above nominal and wins $22$ of $27$ pairwise cells (three scenarios $\times$ three coverage levels $\times$ three competitors---AgACI, NExCP, BCP), including $7/9$ under Scenario~C; two limitations remain. \textbf{(L2)} SA-BCP recovers coverage more slowly than BCP after an abrupt jump (Scenario A): it waits for spatial evidence to accumulate, trailing BCP by up to $\approx 15$pp at 80\% over the first 50 steps (3--5pp at 90--95\%), and AgACI attains the lowest overall Winkler at 80\%/90\% on the jump series---the price of proactive recognition on genuinely novel states with no historical analogue. \textbf{(L3)} Under $t_3$ innovations (Scenario B), conditional coverage on the top-10\% $|Y_t|$ events drops to $0.40/0.44/0.65$, a finite-sample upper-tail limitation shared by all empirical-CDF CP methods that SA-BCP does not resolve.

\noindent\textbf{Extensions and recommendations.} Faster kernel sums (approximate nearest-neighbor, $O(\log n)$) and learned state embeddings \citep{Chen24} could sharpen regime recognition. As guidance: use \textbf{SA-BCP} for volatile targets with recurring regimes needing tight coverage (95\%+) without inflation; \textbf{CF-GARCH} when the series is smooth with a well-specified volatility model and mild under-coverage is acceptable; \textbf{BCP} when post-shock recovery on novel jumps matters more than width.

\noindent\textbf{Conclusion.} SA-BCP resolves the adaptation--stability tradeoff by decoupling temporal inertia from spatial pattern memory through a single interpretable threshold $K$, backed by the guarantees of Section~\ref{sec:theory}. We hope the spatio-temporal decoupling principle proves useful beyond the settings studied here.

\section*{Declarations}

\noindent\textbf{Funding.} No funding was received for conducting this study.

\noindent\textbf{Competing interests.} The authors declare that they have no competing interests.

\noindent\textbf{Author contributions.} Both authors designed the study; Y.-H.\ Fang implemented the method and experiments and drafted the manuscript; C.-Y.\ Lee supervised the work and revised the manuscript. Both authors approved the final version.

\noindent\textbf{Ethics approval and consent.} Not applicable: the study uses only publicly available, non-personal financial and meteorological data.

\noindent\textbf{Data and code availability.} The financial series are publicly available through the \texttt{yfinance} API, and the Jena Climate dataset is provided by the Max Planck Institute for Biogeochemistry; complete preprocessing and per-baseline configurations, the full per-cell result tables, the $K$-sweep and synthetic-rate validation, and code to reproduce every experiment are available in a public repository at \url{https://anonymous.4open.science/r/sabcp-anonymous-release-acml-8DFB}.

\bibliography{ll}

\end{document}